\documentclass[letterpaper, 10 pt, conference]{ieeeconf}  

\IEEEoverridecommandlockouts                              

\usepackage{graphics} 
\usepackage{subfig} 
\usepackage{wrapfig}

\usepackage{amsthm} 
\usepackage{amsmath,amssymb}
\usepackage{graphicx}
\usepackage{tabularx}

\usepackage{multirow, makecell}
\usepackage{diagbox}
\usepackage{slashbox}
\usepackage{rotating}
\usepackage{cite}

\usepackage{url}
\usepackage{bm}
\usepackage[table]{xcolor}

\usepackage{hyperref}
\usepackage{siunitx} 
\usepackage{booktabs}
\usepackage{cleveref}
\usepackage{mathtools} 
\usepackage{upgreek} 

\let\labelindent\relax 
\usepackage{enumitem} 

\usepackage{lipsum}

\usepackage[ruled,vlined]{algorithm2e}

\usepackage[nolist, nohyperlinks]{acronym}
\acrodef{MPC}[MPC]{Model Predictive Control}
\acrodef{QP}[QP]{Quadratic Program}
\acrodef{CBF}[CBF]{Control Barrier Function}

\newcommand{\vx}{{\boldsymbol x}}
\newcommand{\vu}{{\boldsymbol u}}

\newcommand{\vy}{{\boldsymbol y}}
\newcommand{\vz}{{\boldsymbol z}}

\newcommand{\vp}{{\boldsymbol p}} 
\newcommand{\vv}{{\boldsymbol v}} 

\newcommand{\calC}{\mathcal{C}} 

\newcommand{\calX}{\mathcal{X}} 
\newcommand{\calW}{\mathcal{W}} 
\newcommand{\calB}{\mathcal{B}}
\newcommand{\calO}{\mathcal{O}} 
\newcommand{\calS}{\mathcal{S}} 

\newcommand{\R}{\mathbb{R}}
\newcommand{\calU}{\mathcal{U}}

\newcommand{\calH}{\mathcal{H}}
\newcommand{\calI}{\mathcal{I}}
\newcommand{\phib}{\varphi_{\textup{b}}}
\newcommand{\Phib}{\Phi_{\textup{b}}}
\newcommand{\pib}{\pi_{\textup{b}}}
\newcommand{\fb}{f_{\textup{b}}}

\newcommand{\dist}{\operatorname{dist}}

\newtheorem{proposition}{Proposition}

\theoremstyle{definition}
\theoremstyle{definition}
\newtheorem{problem}{Problem}

\theoremstyle{definition}

\theoremstyle{definition}

\theoremstyle{definition}
\newtheorem{assumption}{Assumption}

\usepackage{fontawesome5}

\newif\ifstatementboxes
\statementboxestrue             

\ifstatementboxes
  \usepackage[skins,breakable]{tcolorbox}

  \definecolor{statementblue}{RGB}{239,247,252}
  \definecolor{statementline}{RGB}{210,220,235}

  \definecolor{assumptionblue}{RGB}{242,249,244}
  \definecolor{assumptionline}{RGB}{211,228,216}

  \definecolor{problemblue}{RGB}{243,246,253}
  \definecolor{problemline}{RGB}{215,222,238}

  \definecolor{proofblue}{RGB}{251,253,255}
  \definecolor{proofline}{RGB}{238,242,247}

  \tcbset{
    paperbox/.style={
      enhanced, breakable,
      boxrule=0pt, frame hidden,
      arc=0pt, outer arc=0pt, boxsep=0pt,
      left=7pt, right=6pt, top=6pt, bottom=6pt,
      before skip=6pt, after skip=6pt
    },
    paperstatement/.style={
      paperbox,
      colback=statementblue, colframe=statementblue,
      borderline west={3pt}{0pt}{statementline}
    },
    paperassumption/.style={
      paperbox,
      colback=assumptionblue, colframe=assumptionblue,
      borderline west={3pt}{0pt}{assumptionline}
    },
    paperproblem/.style={
      paperbox,
      colback=problemblue, colframe=problemblue,
      borderline west={3pt}{0pt}{problemline}
    },
    paperproof/.style={
      paperbox,
      colback=proofblue, colframe=proofblue,
      borderline west={3pt}{0pt}{proofline}
    }
  }

  \tcolorboxenvironment{theorem}{paperstatement}
  \tcolorboxenvironment{lemma}{paperstatement}
  \tcolorboxenvironment{proposition}{paperstatement}
  \tcolorboxenvironment{assumption}{paperassumption}
  \tcolorboxenvironment{problem}{paperproblem}
  \tcolorboxenvironment{proof}{paperproof}
\fi
\makeatother

\DeclareCaptionFont{mysize}{\fontsize{8}{9.6}\selectfont}
\title{\LARGE \bf Predictive Semantic Safety: \\From Visual Physical Reasoning to Safety-Critical Control}
\author{Taekyung Kim$^{1, *}$, Salem Fradi$^{2, *}$, Yanning Dai$^{2}$, Mateusz Ostaszewski$^{2}$, Jürgen Schmidhuber$^{2,3,4,5}$
\thanks{$^{*}$These authors contributed equally to this work}
\thanks{$^{1}$Department of Robotics, University of Michigan, Ann Arbor, MI 48109, USA {\tt\footnotesize taekyung@umich.edu} } %
\thanks{$^{2}$Center of Excellence in Generative AI, King Abdullah University of Science and Technology (KAUST), Thuwal, Saudi Arabia.}
\thanks{$^{3}$Dalle Molle Institute for Artificial Intelligence Research (IDSIA), Switzerland. $^{4}$Universit\`a della Svizzera italiana (USI), Switzerland. $^{5}$Scuola universitaria professionale della Svizzera italiana (SUPSI), Switzerland.}
}
\begin{document}
\maketitle
\thispagestyle{empty}
\pagestyle{empty}

\begin{abstract} 
Physical interactions can create future hazards that are not apparent from the robot's current geometric surroundings. We present a framework termed Predictive Semantic Safety (PSS), which connects visual physical reasoning to backup-based safety filtering. A vision-language model (VLM) predicts physical events and their timing or directly predicts object displacements. An explicit motion model converts event hypotheses into object trajectories. Split conformal prediction calibrates position errors jointly across specified objects, observation times, and future times; geometric shape bounds convert the resulting position regions into predicted object occupancy. PSS evaluates a prescribed backup maneuver against this occupancy and derives input-affine constraints for minimally modifying the nominal input while preserving backup feasibility under the robot dynamics and input limits. MuJoCo experiments with a Unitree Go1 consider falling fixtures, impact-driven support loss, and contact propagation. PSS achieves a safe episode rate of 99.3\%, compared with 43.3\% for a Backup Control Barrier Function baseline that only uses current obstacle geometry.
\href{https://www.taekyung.me/pss}{\textcolor{red}{[Project Page]}}\footnote{Project page: \href{https://www.taekyung.me/pss}{https://www.taekyung.me/pss}} \href{https://github.com/AlKeshi/predictive-semantic-safety}{\textcolor{red}{[Code]}} \href{https://youtu.be/p9E79YJ-2Fk}{\textcolor{red}{[Video]}} 
\end{abstract}




\section{INTRODUCTION}
\label{sec:introduction} 
Robots navigating homes, warehouses, and industrial facilities must account
for hazards that develop as their surroundings change. Water spreading
across a corridor can make the floor unsafe to traverse, a growing fire
can block an exit, and an unstable stack of boxes can collapse into an
otherwise clear passage. Assessing these situations requires reasoning
beyond current obstacle geometry to anticipate how physical interactions
may change the safety of the surrounding space. Consider the partially
detached ceiling fixture in Fig.~\ref{fig:concept}. Although the floor
beneath it remains clear, an approaching robot may need to brake before
the fixture falls. Waiting until the hazard enters its path may leave
insufficient control authority to avoid collision. Safety therefore
depends on anticipating physical changes and responding while a safe
maneuver remains feasible. 

Visual physical reasoning provides predictions of how observed objects may move and interact. Physical-interaction benchmarks and learned dynamics models study this capability across a range of scenes and motion types~\cite{bear2021physion,han2022sgnn,yuan2024egode,chow2025physbench}. More recently, vision-language models (VLMs) have shown increasing capability to reason about physical events directly from visual observations. Code-as-World, for example, uses executable world representations to train VLMs for quantitative physical reasoning~\cite{wang2026codeasworlds}. These advances make it possible to anticipate how a physical event may change the scene, but they do not directly specify how the robot should respond. A safety-critical controller must additionally account for the predicted object's spatial extent, the
timing and uncertainty of its predicted motion, and the robot's
dynamics and input limits~\cite{kim_your_2026}.

\begin{figure}[!t]
\centering
\includegraphics[width=\columnwidth]{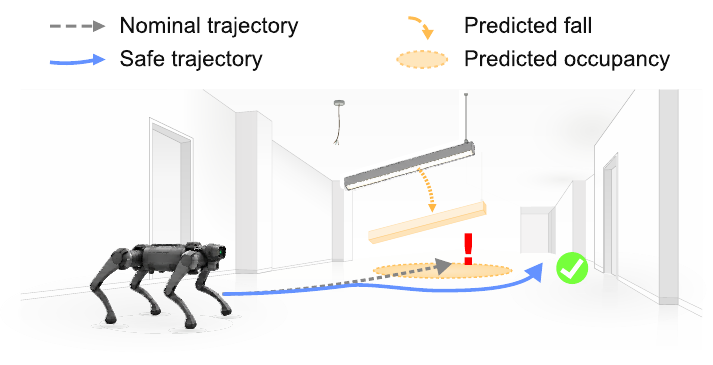}
\caption{Motivating example of predictive semantic safety. A partially
detached ceiling fixture remains above an apparently clear route, but its
anticipated fall may leave the robot with insufficient control authority
to avoid collision.}
\label{fig:concept}
\end{figure} 

To address this problem, we introduce a framework termed
\textbf{\emph{Predictive Semantic Safety}} (PSS), which converts
predicted physical events into time-varying unsafe occupancy for
backup-based safety filtering~\cite{kim_backupbased_2026}.
Semantic reasoning identifies physical mechanisms that may change
object motion, such as support loss, impact, or contact propagation.
Using Code-as-World-VL~\cite{wang2026codeasworlds}, we obtain
object-motion predictions either directly or through an explicit
motion model conditioned on an event hypothesis and its timing.
Split conformal prediction calibrates the resulting position errors,
and geometric shape bounds convert the position regions into
predicted object occupancy (Fig.~\ref{fig:overview}). 

Occupancy coverage alone does not establish that the robot retains
a feasible response. Building on Backup Control Barrier Functions
(CBFs)~\cite{chen_backupcbf_2021} and
OcclusionCBF~\cite{kim2026occlusioncbf}, PSS evaluates a prescribed
backup maneuver against the predicted occupancy and requires its
terminal state to admit a safe continuation. The resulting filter
minimally modifies the nominal input subject to input-affine
constraints that preserve backup feasibility under the robot
dynamics and input limits. The filter can therefore intervene
before predicted object motion obstructs the route, without
imposing a particular nominal navigation policy. 

\begin{figure*}[t] 
\centering
\includegraphics[width=\textwidth]{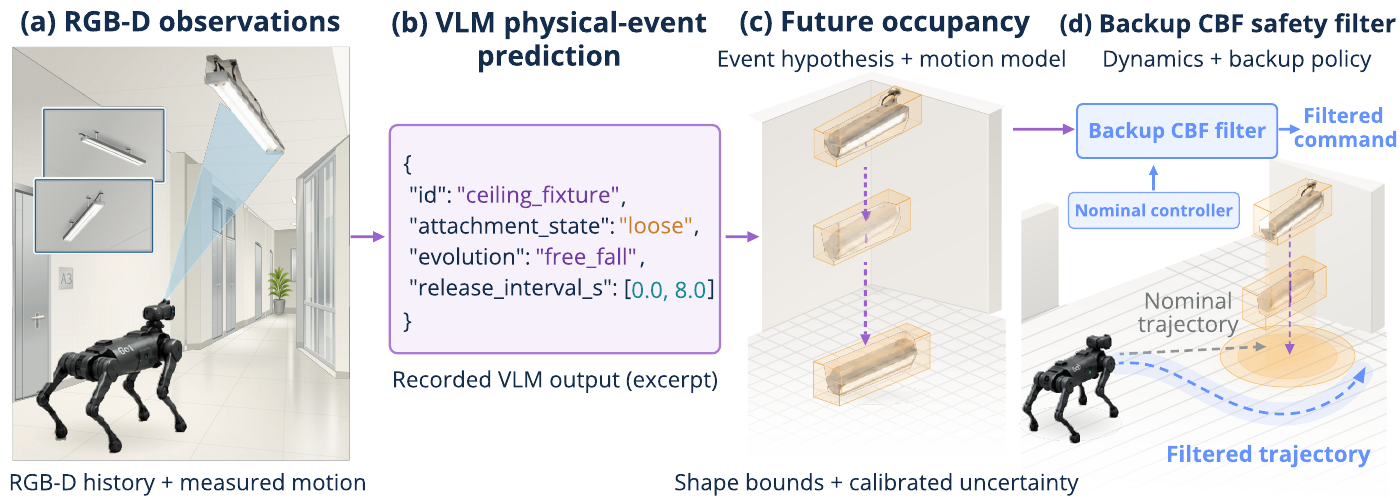}
\caption{Overview of PSS, illustrated for the ceiling fixture. RGB-D observations and measured motion are used
to predict a physical event and its timing with a VLM. The event hypothesis
and an explicit motion model determine future object motion, which is
combined with geometric shape bounds and calibrated uncertainty to obtain
unsafe occupancy. A Backup CBF safety filter evaluates a prescribed backup
maneuver against this occupancy and modifies the nominal command when
necessary. The VLM panel shows an excerpt from a recorded model output.
MuJoCo trials are shown in Fig.~\ref{fig:simulation}.}
\label{fig:overview}
\end{figure*}

\subsection{Related Work}

\textbf{Physical reasoning and semantic safety: }
Visual physical reasoning aims to predict how a scene will evolve from
what a robot observes. Benchmarks test whether models can anticipate
physical interactions and understand physical
behavior~\cite{bear2021physion,chow2025physbench}, while learned dynamics
models predict how objects move and influence one
another~\cite{han2022sgnn,yuan2024egode}. Other approaches learn structured
representations of the physical world from video for prediction and
planning, including V-JEPA~2~\cite{assran2025vjepa2}, which realizes ideas
introduced earlier in Predictability Maximization
(PMAX)~\cite{schmidhuber1993predictable}. Code-as-World uses executable
scene descriptions to train physical
reasoning~\cite{wang2026codeasworlds}. These advances provide a basis for
anticipating hazards, but a forecast alone does not determine how the
robot should respond safely. FORTRESS uses multimodal reasoning to
anticipate failure modes and fallback goals, then plans a dynamics-aware
fallback under inferred semantic
constraints~\cite{ganai2025fortress}. In contrast, PSS represents
predicted object motion as time-varying occupancy sets and constrains
the robot's input to preserve the feasibility of a prescribed backup
maneuver against that occupancy. 

\textbf{Prediction uncertainty: }
Conformal prediction has been used to account for trajectory-prediction
uncertainty in planning and MPC~\cite{lindemann_conformal_2023,
dixit_adaptive_2023,kim_learning_2025a}. Most closely related,
Stamouli et al.~\cite{stamouli2024recursive} construct simultaneous
prediction regions across multiple agents, observation times, and future
times for shrinking-horizon MPC. PSS adopts this calibration principle for
object motion predicted from visually inferred physical events and combines
the resulting position regions with geometric shape bounds to obtain
predicted object occupancy for safety filtering. This calibration accounts for prediction uncertainty, while the backup
formulation separately addresses whether the robot retains a feasible response.

\textbf{Predictive safety filtering: }
Standard CBFs construct safety constraints from the current state
\cite{ames_cbfqp_2017}, while Backup CBFs construct them from a prescribed
backup policy and its induced recoverable set
\cite{chen_backupcbf_2021}. Recent extensions evaluate multiple fallback policies through parallel
finite-horizon rollouts~\cite{kim2026plcbf}. Closest to our control formulation,
OcclusionCBF evaluates a prescribed backup rollout against reachable
occupancy of potentially hidden dynamic obstacles and a terminal
condition~\cite{kim2026occlusioncbf}. 
PSS builds on this construction, but addresses changes in the motion
of observed objects caused by visually inferred physical events.
It calibrates the resulting position errors and evaluates the backup
against the corresponding time-varying occupancy.

\subsection{Contributions} 
Our contributions are summarized as follows.
\begin{itemize}[leftmargin=*,itemsep=1pt,topsep=2pt]
    \item We introduce PSS, which converts visually inferred physical
    events and their timing into time-varying unsafe occupancy sets for
    safety-critical control.
    
    \item We combine conformal calibration with backup-based safety
    filtering to derive input-affine occupancy and terminal
    constraints that preserve backup feasibility under the robot
    dynamics and input limits.

    \item We evaluate PSS in MuJoCo quadruped navigation,
    demonstrating higher safe episode rates than both the CBF
    baselines using current obstacle geometry and the OmniVLA
    navigation policy.
\end{itemize}

\section{PRELIMINARIES}
\label{sec:preliminaries}
\subsection{Robot Dynamics}
Consider a robot with control-affine dynamics
\begin{equation}
\dot{\vx}=f(\vx)+g(\vx)\vu,
\qquad \vx\in\calX,\quad \vu\in\calU,
\label{eq:dynamics}
\end{equation}
where $\calX\subseteq\R^n$ is an open state domain,
$\calU\subset\R^m$ is compact and convex, and $f$ and $g$ are continuously
differentiable. Let $\calW\subseteq\R^d$ be the
workspace. A continuously differentiable map $P:\calX\to\calW$
gives the position $\vp=P(\vx)$ of a fixed robot reference point.
A nominal controller
$\vu_{\textup{nom}}:\calX\times\R_{\geq0}\to\calU$ commands the robot
toward its task objective but is not assumed to satisfy the safety
constraints. It may be any task-achieving policy, including a
reinforcement-learning or vision-language-action (VLA) policy.

\subsection{Control Barrier Functions}
Let $h:\calX\to\R$ be continuously differentiable and define
$\calC=\{\vx\in\calX\mid h(\vx)\geq0\}$. The function $h$ is a
control barrier function if there exists a locally Lipschitz extended
class-$\mathcal K$ function $\alpha:\mathbb{R}\rightarrow\mathbb{R}$ such that~\cite{ames_cbf_2019}
\begin{equation}
\sup_{\vu\in\calU}
\nabla h(\vx)^\top\bigl(f(\vx)+g(\vx)\vu\bigr)
\geq-\alpha\bigl(h(\vx)\bigr),
\quad \forall\vx\in\calX. \nonumber
\end{equation}
A locally Lipschitz controller satisfying the corresponding inequality
renders $\calC$ forward invariant for as long as the solution exists in
$\calX$. However, $h(\vx)\geq0$ alone does not ensure feasibility of the barrier
condition under the input bounds. Backup CBFs address this issue using a
prescribed admissible backup policy. 

\subsection{Backup Control Barrier Functions}
\label{sec:backup_cbf}
For a time-invariant safe set $\calC$, let $\calS_0\subseteq\calC$ be
a terminal set that is forward invariant under a continuously differentiable
state-feedback policy $\pib:\calX\to\calU$. For a backup horizon $T>0$, the recoverable set induced by $\pib$
contains states whose backup trajectories remain in $\calC$ over $[0,T]$
and reach $\calS_0$ at $T$
\cite{chen_backupcbf_2021,kim_backupbased_2026}. Its size depends on the
choice of backup policy. 

The closed-loop backup dynamics are
$\fb(\vx)=f(\vx)+g(\vx)\pib(\vx)$. Their flow $\phib(\vx,s)$ satisfies
\begin{equation}
\frac{\mathrm d}{\mathrm ds}\phib(\vx,s)
=\fb\bigl(\phib(\vx,s)\bigr),
\qquad \phib(\vx,0)=\vx,
\label{eq:backup_flow}
\end{equation}
where $s\geq0$ is the rollout time. The sensitivity Jacobian
$\Phib(\vx,s):=\partial\phib(\vx,s)/\partial\vx$ satisfies
\begin{equation}
\frac{\mathrm d}{\mathrm ds}\Phib(\vx,s)
=\frac{\partial\fb}{\partial\vx}\bigl(\phib(\vx,s)\bigr)\Phib(\vx,s),
\quad \Phib(\vx,0)=I. \nonumber
\end{equation}
Since the backup dynamics are autonomous,
\begin{equation}
\Phib(\vx,s)\fb(\vx)
=\fb\bigl(\phib(\vx,s)\bigr).
\label{eq:flow_identity}
\end{equation}
The Backup CBF filter imposes separate safety conditions along the backup
trajectory and at its terminal state. In the predictive setting considered
here, the backup trajectory must be evaluated against obstacle occupancy at
the corresponding future times, together with a terminal condition that
ensures a safe continuation beyond the prediction horizon. 

\section{PROBLEM FORMULATION}
\label{sec:problem}

Consider planar navigation in $\calW\subseteq\R^2$ for the robot
in~\eqref{eq:dynamics}. Let $\mathcal J$ denote the finite set of objects
included in the safety model, $\calO_j(\tau)\subseteq\calW$ the region
occupied by object $j$ at time $\tau$, and $\vp_j(\tau)\in\calW$ its
reference position. Let $\mathcal R(\vx)\subseteq\calW$ denote the robot's
occupied geometry. RGB-D observations provide object geometry and motion
history used for physical-event prediction. 

\begin{assumption}[Object inventory and geometric bounds]
\label{ass:object_scope}
The object set $\mathcal J$ is specified independently of the prediction
output. For each $j\in\mathcal J$, a known radius $r_j>0$ bounds the
object's planar extent about $\vp_j(\tau)$ for all considered orientations.
\end{assumption}

\begin{assumption}[Exogenous object motion]
\label{ass:exogenous}
Over the prediction horizon, the future motion of each $j\in\mathcal J$
is independent of the candidate robot input evaluated by the safety filter.
\end{assumption}

The safety claim is limited to the objects in $\mathcal J$; detecting
previously unmodeled hazards is outside the scope of the formulation.
Likewise, robot-induced interactions are not covered by
Assumption~\ref{ass:exogenous} and would require action-conditioned
predictions or conservative bounds on their effects.

\begin{problem}[Safety filtering with visual physical reasoning]
\label{prob:predictive_safety}
Given the robot dynamics in~\eqref{eq:dynamics}, the nominal controller
$\vu_{\textup{nom}}$, the prescribed backup policy introduced in
Section~\ref{sec:backup_cbf}, and physical-event predictions for the objects
in $\mathcal J$, construct a safety filter that minimally modifies the
nominal input while accounting for prediction uncertainty and preserving
feasibility of the backup maneuver. Under
Assumptions~\ref{ass:object_scope} and~\ref{ass:exogenous}, the desired
safety condition is
\begin{equation}
\dist\bigl(\mathcal R(\vx(\tau)),\calO_j(\tau)\bigr)\geq\epsilon,
\qquad \forall j\in\mathcal J,\ \forall \tau,
\label{eq:objective}
\end{equation}
where $\epsilon>0$ is the prescribed clearance.
\end{problem} 

\section{PREDICTIVE SEMANTIC SAFETY}
\label{sec:forecasting}
We address Problem~\ref{prob:predictive_safety} by predicting object
motion, constructing calibrated occupancy sets, and evaluating the
prescribed backup maneuver against them (Fig.~\ref{fig:overview}).
The filter modifies the nominal input to preserve backup feasibility.

Let $t_k$ denote the observation time and $\tau\geq t_k$ an absolute
future time. We suppress the fixed $t_k$ except when comparing
predictions from different observation times.

\subsection{Visual Physical Reasoning}

Physical interactions can alter an object's motion before the resulting
hazard is apparent from current geometry. Loss of support can cause a
suspended object to fall, while impact can set a stationary object in
motion. We therefore use visual observations to infer how an object's
motion may change and when that change may occur, \emph{rather than predicting
a semantic risk label alone}.

We use Code-as-World-VL~\cite{wang2026codeasworlds}, a pretrained
VLM, for visual physical reasoning. Given a recent RGB clip and object
motion history measured from RGB-D, a first query interprets the observed
physical interaction. A second structured query uses the same observations
and the first response to identify the object, its attachment state, the
predicted motion mechanism, and its timing. For the ceiling fixture, the
timing is represented by a detachment-time interval. We then use a motion
model to convert this event hypothesis into metric object trajectories,
representing the fixture as supported before detachment, falling after
detachment, and stationary after ground contact.

The predicted detachment-time interval is treated as uncertain. After
verifying that the structured VLM output contains the required fields and
valid numerical values, we sample candidate detachment times within the
interval and generate the corresponding trajectories with the motion
model. For the stack scenarios, the VLM instead predicts object
displacement directly, without a separate detachment-time model. In both
cases, the resulting motion predictions are calibrated in
Section~\ref{sec:conformal}.

\subsection{Event-Conditioned Object Motion}

For object $j\in\mathcal J$, let $m_j$ be the predicted motion mechanism
and $\Theta_j$ a nonempty finite set of motion hypotheses. Each
$\theta\in\Theta_j$ specifies an event realization, such as a detachment
time relative to $t_k$. Given the measured position and linear velocity
$\vz_j\in\R^6$ at $t_k$, the motion model predicts
\begin{equation}
\widehat{\boldsymbol q}_j(t_k+s;\theta)
=\Psi_{m_j}(\vz_j,s;\theta),
\qquad \theta\in\Theta_j,
\label{eq:event_rollout}
\end{equation}
where $s\geq0$ is the look-ahead time and
$\widehat{\boldsymbol q}_j(t_k+s;\theta)\in\R^3$.

Let $\boldsymbol q_j(\tau)\in\R^3$ be the true object reference position
and $\Pi:\R^3\to\R^2$ the planar projection, with
$\vp_j(\tau)=\Pi(\boldsymbol q_j(\tau))$. At $\tau=t_k+s$, define the
mean planar prediction and maximum deviation across hypotheses as
\begin{equation}
\begin{aligned}
\widehat{\vp}_j(\tau)
&=\frac{1}{|\Theta_j|}\sum_{\theta\in\Theta_j}
\Pi\bigl(\widehat{\boldsymbol q}_j(\tau;\theta)\bigr),\\
d_j(\tau)
&=\max_{\theta\in\Theta_j}
\bigl\|
\Pi\bigl(\widehat{\boldsymbol q}_j(\tau;\theta)\bigr)
-\widehat{\vp}_j(\tau)
\bigr\|_2.
\end{aligned}
\label{eq:hypothesis_enclosure}
\end{equation}
By Assumption~\ref{ass:object_scope}, the disk centered at
$\widehat{\vp}_j(\tau)$ with radius $r_j+d_j(\tau)$ contains the
shape disk of each motion hypothesis. The spread $d_j$ does not
account for shared prediction errors or unrepresented event
realizations. For $|\Theta_j|=1$, $d_j(\tau)=0$.

\subsection{Conformal Occupancy Prediction}
\label{sec:conformal}

We use split conformal prediction to calibrate the complete
observation-to-motion prediction procedure
\cite{angelopoulos_conformal_2023,stamouli2024recursive}.
We restore the observation-time argument as
$\widehat{\vp}_j(\tau;t_k)$ and $d_j(\tau;t_k)$.
Let $\calI$ be a fixed finite set of triples $(j,k,\tau)$, with
$j\in\mathcal J$ and $\tau>t_k$, specifying the objects, observation
times, and future times to be covered. The set $\calI$ is fixed
independently of prediction outcomes.

We split the data by complete trajectories into disjoint training
and calibration sets. Each calibration example includes all
observations and object motions over $\calI$. The prediction
procedure and positive finite scales $\sigma_j(\tau;t_k)$ are fixed
independently of the calibration data and remain unchanged during
calibration and testing. The scales are computed from training data.

For each calibration trajectory $c=1,\ldots,N_{\textup{cal}}$, define the
normalized nonconformity score
\begin{equation}
R_c
=
\max_{(j,k,\tau)\in\calI}
\frac{
\bigl\|
\vp_j^{(c)}(\tau)
-\widehat{\vp}_j^{(c)}(\tau;t_k)
\bigr\|_2
}{
\sigma_j(\tau;t_k)
}.
\label{eq:score}
\end{equation}
The normalization accounts for differences in prediction-error scale
across objects and horizons~\cite{yu_stl_2026}. A missing or invalid
required prediction gives $R_c=+\infty$.

Fix the miscoverage level $\delta\in(0,1)$ and suppress dependence on it.
Define the conformal quantile
\begin{equation}
\widehat q
\coloneqq
\operatorname{Quantile}_{1-\delta}
\bigl(
R_1,\ldots,R_{N_{\textup{cal}}},+\infty
\bigr),
\label{eq:quantile}
\end{equation}
where $\operatorname{Quantile}_{1-\delta}$ denotes the
$\left\lceil(N_{\textup{cal}}+1)(1-\delta)\right\rceil$-th smallest
element of its arguments.
Let $\calB(r)=\{\boldsymbol e\in\R^2\mid\|\boldsymbol e\|_2\leq r\}$
be the closed planar ball. With calibrated radius
$\rho_j(\tau;t_k)\coloneqq\widehat q\,\sigma_j(\tau;t_k)$, define
\begin{equation}
\mathcal T_j(\tau;t_k)
\coloneqq
\widehat{\vp}_j(\tau;t_k)+\calB\bigl(\rho_j(\tau;t_k)\bigr).
\label{eq:prediction_region}
\end{equation}
If the test prediction is invalid or $\widehat q=+\infty$, set
$\mathcal T_j(\tau;t_k)=\R^2$.

By Assumption~\ref{ass:object_scope},
$\calO_j(\tau)\subseteq\vp_j(\tau)+\calB(r_j)$.
Combining this shape bound with the conformal position region and
hypothesis spread gives the predicted object occupancy
\begin{equation}
\widehat{\calO}_j(\tau;t_k)
\coloneqq
\mathcal T_j(\tau;t_k)
\oplus
\calB\bigl(r_j+d_j(\tau;t_k)\bigr),
\label{eq:inflation}
\end{equation}
where $\oplus$ denotes the Minkowski sum. If the required prediction is
invalid or no motion hypothesis is available, we set
$\widehat{\calO}_j(\tau;t_k)=\R^2$.

\begin{proposition}[Simultaneous occupancy coverage]
\label{prop:containment}
Suppose the prediction procedure and normalization scales are fixed
independently of the calibration data, the calibration and test
trajectories, including their observations, are exchangeable, and
Assumption~\ref{ass:object_scope} holds over $\calI$. Then
\begin{equation}
\mathbb P\!\left(
\calO_j(\tau)\subseteq\widehat{\calO}_j(\tau;t_k),
\quad \forall (j,k,\tau)\in\calI
\right)
\geq 1-\delta.
\label{eq:coverage}
\end{equation}
\end{proposition}

\begin{proof}
The split-conformal rank argument applied to~\eqref{eq:score} gives
$\vp_j(\tau)\in\mathcal T_j(\tau;t_k)$ simultaneously over $\calI$
with probability at least $1-\delta$. On this event,
Assumption~\ref{ass:object_scope} and $d_j(\tau;t_k)\geq0$ imply
$\calO_j(\tau)\subseteq\mathcal T_j(\tau;t_k)\oplus\calB(r_j)
\subseteq\widehat{\calO}_j(\tau;t_k)$.
\end{proof}

The guarantee is marginal over exchangeable calibration and test
trajectories, not conditional on a scene or an accepted update.
Since~\eqref{eq:score} measures error from the mean prediction,
$d_j$ does not justify reducing $\rho_j$. Moreover,
Proposition~\ref{prop:containment} provides coverage only for the objects
and times included in $\calI$; it does not imply coverage between those
times or feasibility of the robot's response.

\subsection{Backup Rollout and Terminal Set}
\label{sec:recoverability}

Following OcclusionCBF~\cite{kim2026occlusioncbf}, we require the
prescribed backup to avoid the predicted occupancy and reach a terminal
set that admits a safe continuation.

Between accepted updates, fix the prediction and suppress $t_k$.
For current time $t$ and look-ahead time $s\geq0$, let
$\vy=\phib(\vx,s)$. Choose $r_R>0$ such that
$\mathcal R(\vy)\subseteq P(\vy)+\calB(r_R)$ for all relevant states.
Define the collision-inflated predicted occupancy
\begin{equation}
\widehat{\calH}^{(j)}(t,s)
\coloneqq
\widehat{\calO}_j(t+s)
\oplus
\calB(r_R+\epsilon).
\label{eq:collision_occupancy}
\end{equation}
For a valid prediction with finite $\rho_j$, the disk representation
in~\eqref{eq:inflation} gives the separation margin
\begin{equation}
\begin{split}
h_j^C(\vy,t,s)
={}&
\bigl\|
P(\vy)-\widehat{\vp}_j(t+s)
\bigr\|_2
-r_R-r_j\\
&-d_j(t+s)-\rho_j(t+s)-\epsilon.
\end{split}
\label{eq:distance}
\end{equation}
Nonnegativity of $h_j^C$ places $P(\vy)$ outside the interior of
$\widehat{\calH}^{(j)}(t,s)$ and ensures clearance $\epsilon$ between
$\mathcal R(\vy)$ and $\widehat{\calO}_j(t+s)$.
The margin is differentiable in $\vy$ when
$P(\vy)\ne\widehat{\vp}_j(t+s)$. Its time dependence is through $t+s$,
so $\partial h_j^C/\partial t=\partial h_j^C/\partial s$ wherever the prediction is differentiable.
Choose $t_f$ within the common prediction interval and define the
remaining backup horizon $T(t)\coloneqq t_f-t$ for $t\leq t_f$.
The terminal time and backup policy are fixed between accepted updates.

Let
$\calS_0(\tau)\coloneqq
\{\vy\in\calX\mid h_j^S(\vy,\tau)\geq0,\ \forall j\in\mathcal J\}$
for continuously differentiable terminal margins $h_j^S$.
Membership in $\calS_0(t_f)$ must ensure separation at $t_f$ and a safe
continuation under $\pib$ thereafter.

For each absolute future time $\tau\in[t,t_f]$, define the occupancy and
terminal margins evaluated along the backup flow as
\begin{equation}
\begin{aligned}
\xi_j^C(\vx,t;\tau)
&\coloneqq
h_j^C\bigl(\phib(\vx,\tau-t),t,\tau-t\bigr),\\
\eta_j(\vx,t)
&\coloneqq
h_j^S\bigl(\phib(\vx,T(t)),t_f\bigr).
\end{aligned}
\label{eq:backup_margin}
\end{equation}
The resulting recoverable set is
\begin{equation}
\begin{split}
\calS(t)
\coloneqq
\bigl\{\vx\in\calX\ \big|\ &
\xi_j^C(\vx,t;\tau)\geq0,
\quad \forall j\in\mathcal J,\ \forall\tau\in[t,t_f],\\
&
\eta_j(\vx,t)\geq0,
\quad \forall j\in\mathcal J
\bigr\}.
\end{split}
\label{eq:recoverable}
\end{equation}
The set $\calS(t)$ contains states whose backup avoids the predicted
occupancy through $t_f$ and reaches $\calS_0(t_f)$.
In Fig.~\ref{fig:recoverability}, \emph{Semantic-Off} evaluates the
backup against current geometry, whereas \emph{Semantic-On} uses
predicted future occupancy to intervene before the backup becomes
infeasible.

\begin{figure}[t]
\centering
\includegraphics[width=\columnwidth]{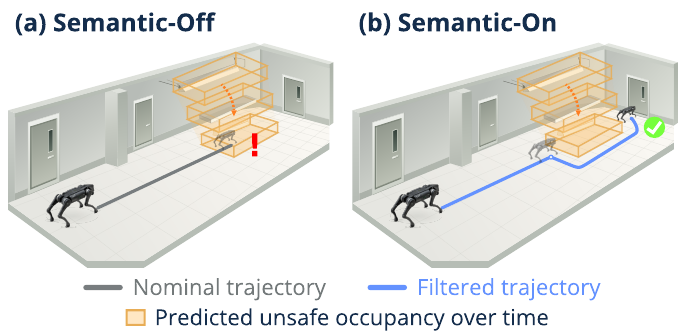}
\caption{Backup feasibility with and without physical-event prediction.
(a) \emph{Semantic-Off} evaluates the backup using current obstacle
geometry, so the resulting motion may continue toward the object's future
occupancy. (b) \emph{Semantic-On} evaluates the same backup against
predicted future occupancy and modifies the nominal motion while the
backup remains feasible. Predicted occupancy is shown in both panels but
is used only by \emph{Semantic-On}. Generated conceptual illustration;
geometry, trajectories, and timing are schematic.}
\label{fig:recoverability}
\end{figure}

\subsection{Input-Affine Safety Filtering}

We impose separate CBF conditions on the occupancy and terminal margins.
For fixed $\tau$, let $s=\tau-t$, $\vy=\phib(\vx,s)$, and
$\vy_f=\phib(\vx,T(t))$. Define
\begin{equation}
\begin{aligned}
A_j^C
&\coloneqq
\bigl(\nabla_{\vy}h_j^C\bigr)^\top\Phib(\vx,s)g(\vx),\\
A_j^S
&\coloneqq
\bigl(\nabla_{\vy}h_j^S\bigr)^\top\Phib(\vx,T(t))g(\vx),
\end{aligned}
\label{eq:affine_coefficients}
\end{equation}
where the gradients are evaluated at $(\vy,t,s)$ and $(\vy_f,t_f)$,
respectively. Since $\tau$ and $t_f$ are fixed,
$\mathrm ds/\mathrm dt=\dot T=-1$.
Using $\partial_t h_j^C=\partial_s h_j^C$ and
\eqref{eq:flow_identity}, differentiation of~\eqref{eq:backup_margin}
gives $\dot\xi_j^C=A_j^C(\vu-\pib(\vx))$ and
$\dot\eta_j=A_j^S(\vu-\pib(\vx))$.
We suppress the arguments of $A_j^C$, $A_j^S$, $\xi_j^C$, and $\eta_j$
when their evaluation points are clear. The safety filter is then
\begin{equation}
\begin{aligned}
\vu^\star(\vx,t)
={}&
\arg\min_{\vu\in\calU}
\frac12
\bigl\|
\vu-\vu_{\textup{nom}}(\vx,t)
\bigr\|_W^2\\
\textup{s.t.}\quad
&
A_j^C\bigl(\vu-\pib(\vx)\bigr)
\geq
-\alpha_C(\xi_j^C),
\quad
\forall j,\tau,\\
&
A_j^S\bigl(\vu-\pib(\vx)\bigr)
\geq
-\alpha_S(\eta_j),
\quad
\forall j,
\end{aligned}
\label{eq:qp}
\end{equation}
where $j\in\mathcal J$, $\tau\in[t,t_f]$, $W\succ0$, and
$\alpha_C,\alpha_S$ are locally Lipschitz extended class-$\mathcal K$
functions. The predictive constraints are affine in $\vu$. Such QP-based safety filters can be solved at millisecond-level runtime
in robotic systems~\cite{kim2026plcbf}. With finitely
many enforced times, affine static-obstacle constraints, and polyhedral
$\calU$, \eqref{eq:qp} is a QP; for a general compact convex $\calU$, it
remains a convex optimization problem.

For $\vx\in\calS(t)$, $\vu=\pib(\vx)$ satisfies both predictive
constraint families. Under continuous enforcement and suitable
regularity, scalar comparison preserves initially nonnegative margins
between updates. Additional static-obstacle CBF constraints must remain
jointly feasible; higher-order CBFs may be used when required by
relative degree~\cite{xiao_hocbf_2019}.

\subsection{Prediction Updates and Practical Implementation}

Predictions retain their observation times, so delays shorten the
remaining horizon without shifting the predicted event. We accept a
joint update for all objects only after checking object identities,
the prediction interval, current-state membership in the updated
$\calS(t)$, and feasibility of the full filter, including static-obstacle
constraints. After rejection, the previous prediction remains active
only while its own object scope, time interval, and filter feasibility
remain valid.

A missing prediction is not treated as an absent obstacle. If no valid
prediction remains, the controller executes an emergency response,
which does not itself certify safety outside $\calS(t)$. In implementation, the rollout constraints are enforced at finitely many
times in $[t,t_f]$.

\begin{figure*}[t]
\centering
\includegraphics[width=0.95\textwidth]{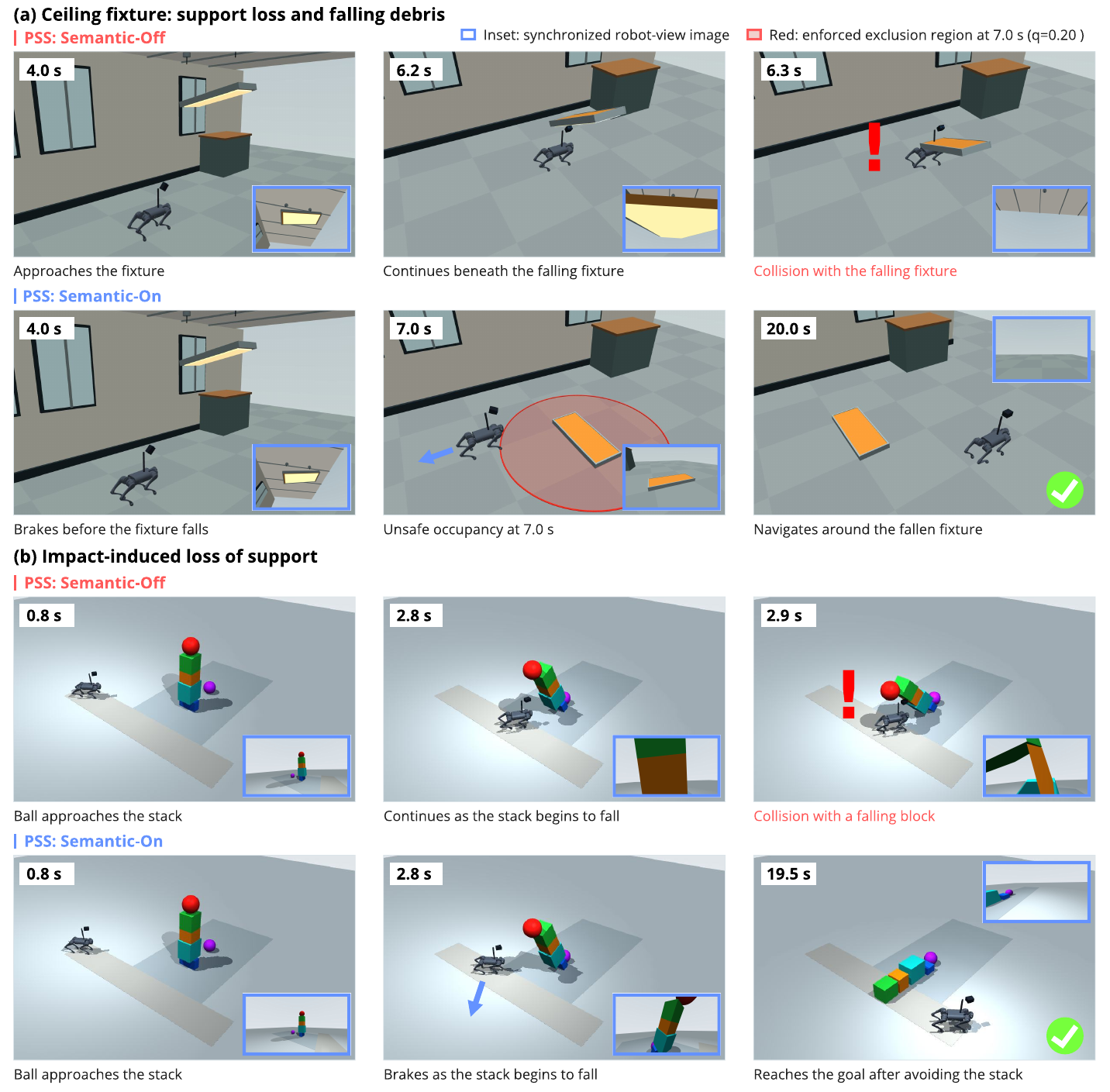}
\caption{Selected MuJoCo trials with physical-event prediction disabled
(Semantic-Off) and enabled (Semantic-On).
(a) Ceiling fixture: Semantic-Off continues beneath the falling fixture
and collides, whereas Semantic-On brakes before full detachment and later
navigates around the debris. The red region shows the unsafe occupancy
used by the filter at 7.0~s.
(b) Impact-induced loss of support: Semantic-Off continues as the stack
begins to fall and is struck by a block, whereas Semantic-On brakes and
subsequently reaches the goal. Insets show synchronized robot-view images.
These qualitative trials are separate from the aggregate benchmark.}
\label{fig:simulation}
\end{figure*}

\section{RESULTS}
\label{sec:experiments}

We evaluate PSS in quadruped navigation with hazards arising from support
loss, impact, and contact propagation, where the resulting future unsafe
occupancy is not evident from current scene geometry alone. The experiments
examine whether anticipating changes in object motion improves collision
avoidance over filters that rely on current obstacle geometry. We report
safe episode rates and use selected trials to examine when the robot brakes
and how it resumes navigation after the event.

\subsection{Experimental Setup}

\textbf{Robot and controller setup: }
We simulate a Unitree Go1 in MuJoCo. PSS uses onboard RGB-D observations
for visual physical reasoning and metric motion estimation. The safety
filter uses the planar double-integrator model
$\vx=[\vp^\top,\vv^\top]^\top$,
$\dot\vp=\vv$, and $\dot\vv=\vu$, where $\vu$ is planar acceleration.
The prescribed backup policy is
\[
\pib(\vx)=\operatorname{sat}_{a_{\max}}
\bigl(-k_{\textup{b}}\vv\bigr),
\]
where $k_{\textup{b}}=2.5\,\mathrm{s}^{-1}$ and saturation is applied
componentwise at $a_{\max}=2.0\,\mathrm{m/s^2}$.
The implementation uses $r_R=0.403$~m, $\epsilon=0.05$~m,
$\alpha_C(h)=\alpha_S(h)=3h$, and $W=I$.
The QP is solved using OSQP without slack variables. The nominal acceleration is proportional to the difference between
a goal-directed velocity reference and the previously commanded
velocity, with gain $2.5\,\mathrm{s}^{-1}$ and componentwise limits
of $1.5\,\mathrm{m/s^2}$. The filtered acceleration is integrated to
obtain the velocity command supplied to the locomotion policy. The reduced-order model does not by itself certify full-body
tracking. Learned-model inference is performed on one NVIDIA A100 GPU.

\textbf{Prediction and sampling: }
We use pretrained Code-as-World-VL-9B without further fine-tuning.
The safety filter and locomotion policy update every 0.02~s, and
RGB-D acquisition is scheduled every 0.05~s. The prediction runners
request updates at 1~s intervals. For the ceiling fixture, trajectories
are computed over 8~s at 0.1~s intervals using up to 17 evenly spaced
detachment-time hypotheses. For the stack and domino scenarios, the
VLM predicts five future positions at 0.5~s intervals over 2.5~s.
The constraint grid includes the prediction nodes and subdivides
intervals to keep their spacing at most 0.15~s. For conformal calibration, we set $\delta=0.05$ and use $N_{\textup{cal}}=1000$. Following
\cite{stamouli2024recursive},
$\sigma_j(\tau;t_k)$ is set to the maximum training-set position error
at the corresponding prediction time.

\textbf{Scenarios: }
We consider a ceiling fixture that loses support, a rolling ball that
destabilizes a stack, and contact propagation through falling dominoes.
The ceiling fixture requires reasoning about a possible fall while the
space beneath it remains clear. The stack and domino scenarios involve
motion induced by interactions among objects and are inspired by
Physion~\cite{bear2021physion}. These cases test whether the robot responds
to the predicted consequences of a physical interaction rather than only
to obstacles already in its path. Figure~\ref{fig:simulation} shows
selected ceiling-fixture and stack trials.

\subsection{Evaluation Protocol and Compared Methods}

The benchmark contains 30 physical cases, with 10 per scenario type and
five variations of the initial robot position and yaw. This gives
150 configurations per method and 750 episodes across five methods,
including hazardous events and benign controls. The primary metric is the
\emph{safe episode rate}: an episode is counted as safe if the robot
avoids collisions with hazards and fixed obstacles throughout the
evaluation horizon. Stopping
without collision counts as a safe outcome, even if the robot does not
reach the goal.

We compare \emph{Plain CBF}, \emph{OmniVLA},
\emph{Backup CBF (Semantic-Off)}, \emph{PSS without the conformal margin},
and \emph{PSS}. Plain CBF and Backup CBF use current simulator-measured
geometry at 50~Hz. An object enters the obstacle map when its lowest
measured point reaches a height of 0.55~m and remains tracked afterward.
Backup CBF uses the same braking policy, flow sensitivities, and
numerical terminal condition as PSS, but holds measured obstacle
geometry fixed during each backup rollout. PSS instead evaluates the
backup against predicted future occupancy. In the qualitative
comparisons, \emph{Semantic-On} denotes PSS with physical-event
prediction enabled.

OmniVLA~\cite{hirose2025omnivla} uses the unmodified
\texttt{NHirose/omnivla-original} checkpoint with the action head and
pose projector from training step 120,000. Its inputs are the current
forward RGB image, the goal pose relative to the robot, and the
instruction ``move toward the specified goal while avoiding obstacles.''
The model predicts eight waypoints; the fifth is converted to
forward-speed and yaw-rate commands using the official decoder, with
limits of 0.3~m/s and 0.3~rad/s, respectively. Inference is requested
at a nominal rate of 3~Hz in simulation time, and the latest command
is held between updates. These commands are tracked by the shared
Go1 locomotion policy without an additional CBF filter.

\subsection{Safety Performance}

Figure~\ref{fig:benchmark_safety} reports the safe episode rate across
all three scenario types. PSS achieves 99.3\%, compared with 20.0\%
for Plain CBF, 37.3\% for OmniVLA, and 43.3\% for Backup CBF.
The improvement over Backup CBF is 56.0 percentage points. Although
Backup CBF improves on Plain CBF, evaluating a braking maneuver against
current geometry does not account for the changes in object motion
caused by the physical event. The comparison supports evaluating the
backup against predicted future occupancy, rather than relying on the
availability of a braking maneuver alone.

PSS without the conformal margin achieves a safe episode rate of 98.7\%,
corresponding to 148 safe episodes out of 150, compared with 149 for PSS.
Both variants substantially outperform the CBF baselines that use current
geometry, while their difference amounts to one safe episode in this
benchmark. The empirical safe episode rate measures collision avoidance;
it is distinct from the occupancy-coverage probability in
Proposition~\ref{prop:containment}.

\begin{figure}[t]
\centering
\includegraphics[width=\columnwidth]{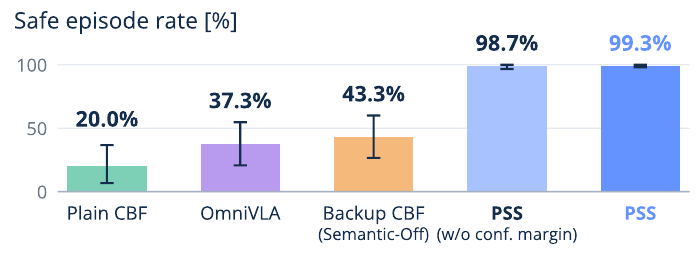}
\caption{Safe episode rates over 750 episodes, with 150 episodes per
method. Error bars show 95\% bootstrap confidence intervals obtained by
resampling physical cases while grouping the five pose variations of
each case. Stopping without collision counts as a safe outcome; goal
completion is not required.}
\label{fig:benchmark_safety}
\end{figure}

\subsection{Intervention Timing and Navigation}

\textbf{Ceiling fixture: }
Figure~\ref{fig:simulation}(a) illustrates the distinction between
current obstacle geometry and predicted unsafe occupancy. With
physical-event prediction disabled, the robot approaches the partially
detached fixture, continues beneath it as it falls, and collides at
6.3~s. PSS instead begins braking while the fixture is still suspended.
The robot remains outside the unsafe occupancy shown at 7.0~s and later
navigates around the fallen fixture, as shown at 20.0~s. Thus, the robot
responds to the anticipated fall before the object obstructs its route,
then continues navigation around the resulting debris.

\begin{figure}[t]
\centering
\includegraphics[width=\columnwidth]{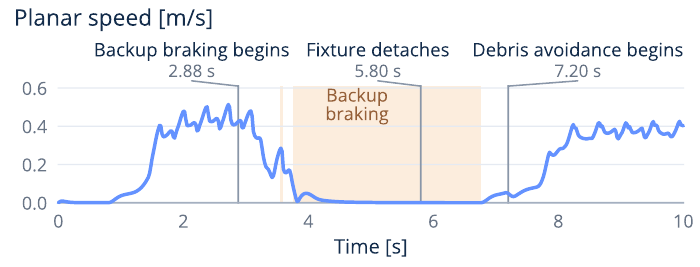}
\caption{Planar speed during the ceiling-fixture trial. Backup braking
begins at 2.88~s, before full fixture detachment at 5.80~s. The robot
remains nearly stationary through detachment and resumes motion during
debris avoidance, which begins at 7.20~s. Annotations mark the three
events.}
\label{fig:ceiling_speed}
\end{figure}

\textbf{Impact-induced loss of support: }
In Fig.~\ref{fig:simulation}(b), a rolling ball destabilizes the stack.
At 2.8~s, Semantic-Off continues as the stack begins to fall and is struck
by a block shortly afterward. Semantic-On brakes as the stack begins to
fall, avoids the falling blocks, and reaches the goal by the 19.5~s
frame. Here, the hazard is not only the approaching ball but also the
motion it induces in the initially stationary stack. The comparison
illustrates the role of predicting the consequences of the impact when
evaluating the robot's backup maneuver. These selected trials are
separate from the aggregate benchmark.

\textbf{Braking before detachment: }
Figure~\ref{fig:ceiling_speed} shows the planar speed during the
ceiling-fixture trial. Backup braking begins at 2.88~s, which is 2.92~s
before full detachment at 5.80~s. The speed decreases to nearly zero by
approximately 4.0~s, and the robot remains nearly stationary through
detachment. Motion resumes during debris avoidance, which begins at
7.20~s. The speed trace therefore confirms that braking precedes full
detachment, while the subsequent acceleration and the trajectory in
Fig.~\ref{fig:simulation}(a) show continued navigation after the event.

\section{CONCLUSION}

We presented PSS, a framework that connects visual physical reasoning
to safety-critical control. PSS converts predicted physical events
into time-varying unsafe occupancy and combines conformal calibration
with backup-based safety filtering. By evaluating a prescribed backup
maneuver against the predicted occupancy, the filter minimally modifies
the nominal input to preserve backup feasibility under the robot
dynamics and input limits. MuJoCo quadruped experiments involving
support loss, impact, and contact propagation demonstrate improved
safe episode rates over filters that rely on current obstacle geometry.
Selected trials show braking before full fixture detachment and
subsequent navigation around the debris, illustrating the importance
of anticipating physical events while a safe response remains feasible.

Future work will extend PSS to robot-induced interactions through
action-conditioned object-motion prediction. We will also incorporate
bounds on locomotion tracking errors into backup evaluation and
validate the framework on physical robots.

\addtolength{\textheight}{0 cm}   

\bibliographystyle{IEEEtran}
\typeout{}
\bibliography{references.bib}

@inproceedings{ames_cbf_2019,
  author        = {Ames, Aaron D. and Coogan, Samuel and Egerstedt, Magnus and Notomista, Gennaro and Sreenath, Koushil and Tabuada, Paulo},
  title         = {Control Barrier Functions: Theory and Applications},
  booktitle     = {European Control Conference (ECC)},
  pages         = {3420--3431},
  year          = {2019},
  doi           = {10.23919/ECC.2019.8796030}
}

@article{ames_cbfqp_2017,
  author        = {Ames, Aaron D. and Xu, Xiangru and Grizzle, Jessy W. and Tabuada, Paulo},
  title         = {Control Barrier Function Based Quadratic Programs for Safety Critical Systems},
  journal       = {IEEE Transactions on Automatic Control},
  volume        = {62},
  number        = {8},
  pages         = {3861--3876},
  year          = {2017},
  doi           = {10.1109/TAC.2016.2638961}
}

@inproceedings{chen_backupcbf_2021,
  author        = {Chen, Yuxiao and Jankovic, Mrdjan and Santillo, Mario and Ames, Aaron D.},
  title         = {Backup Control Barrier Functions: Formulation and Comparative Study},
  booktitle     = {IEEE Conference on Decision and Control (CDC)},
  pages         = {6835--6841},
  year          = {2021},
  doi           = {10.1109/CDC45484.2021.9683111}
}

@inproceedings{kim_backupbased_2026,
	title = {Backup-{Based} {Safety} {Filters}: {A} {Comparative} {Review} of {Backup} {CBF}, {Model} {Predictive} {Shielding}, and gatekeeper},
	shorttitle = {Backup-{Based} {Safety} {Filters}},
	doi = {10.48550/arXiv.2604.02401},
	urldate = {2026-04-06},
	booktitle = {{IEEE} {Conference} on {Decision} and {Control} ({CDC})},
	author = {Kim, Taekyung and Menon, Aswin D. and Trivedi, Akshunn and Panagou, Dimitra},
	year = {2026},
}

@inproceedings{xiao_hocbf_2019,
  author        = {Xiao, Wei and Belta, Calin},
  title         = {Control Barrier Functions for Systems with High Relative Degree},
  booktitle     = {IEEE Conference on Decision and Control (CDC)},
  pages         = {474--479},
  year          = {2019},
  doi           = {10.1109/CDC40024.2019.9029455}
}

@inproceedings{bear2021physion,
  author        = {Bear, Daniel M. and Wang, Elias and Mrowca, Damian and Binder, Felix and Tung, Hsiao-Yu Fish and Pramod, R. T. and Holdaway, Cameron and Tao, Sirui and Smith, Kevin and Sun, Fan-Yun and Li, Fei-Fei and Kanwisher, Nancy and Tenenbaum, Joshua B. and Yamins, Daniel L. K. and Fan, Judith},
  title         = {Physion: Evaluating Physical Prediction from Vision in Humans and Machines},
  booktitle     = {Neural Information Processing Systems (NeurIPS), Datasets and Benchmarks Track},
  year          = {2021}
}

@inproceedings{chow2025physbench,
  author        = {Chow, Wei and Mao, Jiageng and Li, Boyi and Seita, Daniel and Guizilini, Vitor and Wang, Yue},
  title         = {{PhysBench}: Benchmarking and Enhancing Vision-Language Models for Physical World Understanding},
  booktitle     = {International Conference on Learning Representations (ICLR)},
  pages         = {97959--98108},
  year          = {2025}
}

@article{schmidhuber1993predictable,
  author        = {Schmidhuber, J{\"u}rgen and Prelinger, Daniel},
  title         = {Discovering Predictable Classifications},
  journal       = {Neural Computation},
  volume        = {5},
  number        = {4},
  pages         = {625--635},
  year          = {1993},
  doi           = {10.1162/neco.1993.5.4.625},
  note          = {Originally Technical Report CU-CS-626-92, University of Colorado, 1992}
}

@misc{assran2025vjepa2,
  author        = {Assran, Mido and Bardes, Adrien and Fan, David and Garrido, Quentin and Howes, Russell and Komeili, Mojtaba and Muckley, Matthew and Rizvi, Ammar and Roberts, Claire and Sinha, Koustuv and Zholus, Artem and Arnaud, Sergio and Gejji, Abha and Martin, Ada and Hogan, Francois Robert and Dugas, Daniel and Bojanowski, Piotr and Khalidov, Vasil and Labatut, Patrick and Massa, Francisco and Szafraniec, Marc and Krishnakumar, Kapil and Li, Yong and Ma, Xiaodong and Chandar, Sarath and Meier, Franziska and LeCun, Yann and Rabbat, Michael and Ballas, Nicolas},
  title         = {{V-JEPA} 2: Self-Supervised Video Models Enable Understanding, Prediction and Planning},
  howpublished  = {arXiv preprint arXiv:2506.09985},
  year          = {2025},
  doi           = {10.48550/arXiv.2506.09985},
  eprint        = {2506.09985},
  archiveprefix = {arXiv},
  primaryclass  = {cs.AI}
}

@misc{wang2026codeasworlds,
  author        = {Wang, Hanyang and Cai, Yimo and Chen, Weiliang and Chi, Jiawei and Sun, Haowen and Dai, Qiyu and Hung, Yi-Hsin and Guo, Xingzhuo and Ren, Jinshan and Yao, Runmao and Liu, Ziwei and Long, Mingsheng and Duan, Yueqi and Gao, Jun and Lyu, Jiangran and Liu, Fangfu and Wu, Jialong},
  title         = {{Code as Worlds}: Agentic Discovery of Executable World Representations for Physical Reasoning},
  howpublished  = {arXiv preprint arXiv:2608.27549},
  year          = {2026},
  doi           = {10.48550/arXiv.2608.27549},
  eprint        = {2608.27549},
  archiveprefix = {arXiv},
  primaryclass  = {cs.CV}
}

@article{lindemann_conformal_2023,
  author        = {Lindemann, Lars and Cleaveland, Matthew and Shim, Gihyun and Pappas, George J.},
  title         = {Safe Planning in Dynamic Environments Using Conformal Prediction},
  journal       = {IEEE Robotics and Automation Letters},
  volume        = {8},
  number        = {8},
  pages         = {5116--5123},
  year          = {2023},
  doi           = {10.1109/LRA.2023.3292071}
}

@inproceedings{dixit_adaptive_2023,
  author        = {Dixit, Anushri and Lindemann, Lars and Wei, Skylar X and Cleaveland, Matthew and Pappas, George J. and Burdick, Joel W.},
  title         = {Adaptive Conformal Prediction for Motion Planning among Dynamic Agents},
  booktitle     = {Learning for Dynamics and Control Conference (L4DC)},
  pages         = {300--314},
  year          = {2023}
}

@article{angelopoulos_conformal_2023,
  author        = {Angelopoulos, Anastasios N. and Bates, Stephen},
  title         = {Conformal Prediction: A Gentle Introduction},
  journal       = {Foundations and Trends in Machine Learning},
  volume        = {16},
  number        = {4},
  pages         = {494--591},
  year          = {2023},
  doi           = {10.1561/2200000101}
}

@inproceedings{han2022sgnn,
  author        = {Han, Jiaqi and Huang, Wenbing and Ma, Hengbo and Li, Jiachen and Tenenbaum, Joshua B. and Gan, Chuang},
  title         = {Learning Physical Dynamics with Subequivariant Graph Neural Networks},
  booktitle     = {Neural Information Processing Systems (NeurIPS)},
  pages         = {26256--26268},
  year          = {2022},
  doi           = {10.52202/068431-1904}
}

@inproceedings{yuan2024egode,
  author        = {Yuan, Jingyang and Sun, Gongbo and Xiao, Zhiping and Zhou, Hang and Luo, Xiao and Luo, Junyu and Zhao, Yusheng and Ju, Wei and Zhang, Ming},
  title         = {{EGODE}: An Event-attended Graph {ODE} Framework for Modeling Rigid Dynamics},
  booktitle     = {Neural Information Processing Systems (NeurIPS)},
  pages         = {59093--59118},
  year          = {2024},
  doi           = {10.52202/079017-1885}
}

@inproceedings{ganai2025fortress,
  author        = {Ganai, Milan and Sinha, Rohan and Agia, Christopher and Morton, Daniel and Di Lillo, Luigi and Pavone, Marco},
  title         = {Real-Time Out-of-Distribution Failure Prevention via Multi-Modal Reasoning},
  booktitle     = {Conference on Robot Learning (CoRL)},
  pages         = {283--308},
  year          = {2025}
}

@inproceedings{stamouli2024recursive,
  author        = {Stamouli, Charis and Lindemann, Lars and Pappas, George J.},
  title         = {Recursively Feasible Shrinking-Horizon {MPC} in Dynamic Environments with Conformal Prediction Guarantees},
  booktitle     = {Learning for Dynamics and Control Conference (L4DC)},
  pages         = {1330--1342},
  year          = {2024}
}

@misc{kim2026occlusioncbf,
  author        = {Kim, Taekyung and Park, Hun Kuk and Wada, Renya and Atanasov, Nikolay and Koga, Shumon and Panagou, Dimitra},
  title         = {{OcclusionCBF}: Backup Control Barrier Functions for Safe Navigation Among Hidden Dynamic Obstacles},
  howpublished  = {arXiv preprint arXiv:2609.06342},
  year          = {2026},
  doi           = {10.48550/arXiv.2609.06342},
  eprint        = {2609.06342},
  archiveprefix = {arXiv}
}

@article{yu_stl_2026,
  author        = {Yu, Xinyi and Zhao, Yiqi and Yin, Xiang and Lindemann, Lars},
  title         = {Signal temporal logic control synthesis among uncontrollable dynamic agents with conformal prediction},
  journal       = {Automatica},
  volume        = {183},
  pages         = {112616},
  year          = {2026},
  doi           = {10.1016/j.automatica.2025.112616}
}

@inproceedings{kim_your_2026,
	title = {Is {Your} {Safe} {Controller} {Actually} {Safe}? {A} {Critical} {Review} of {CBF} {Tautologies} and {Hidden} {Assumptions}},
	shorttitle = {Is {Your} {Safe} {Controller} {Actually} {Safe}?},
	language = {en},
	urldate = {2026},
	booktitle = {arXiv preprint arXiv:2603.06954},
	author = {Kim, Taekyung},
	year = {2026},
}

@inproceedings{kim_learning_2025a,
  title = {Learning to {{Adapt Control Barrier Functions Under Epistemic}} and {{Aleatoric Uncertainty}}},
  booktitle = {{{arXiv}} Preprint {{arXiv}}:2504.03038},
  author = {Kim, Taekyung and Kee, Robin Inho and Panagou, Dimitra},
  year = 2026,
  eprint = {2504.03038},
  urldate = {2026},
}

@inproceedings{hirose2025omnivla,
  title={{OmniVLA}: An Omni-Modal Vision-Language-Action Model
         for Robot Navigation},
  author={Hirose, Noriaki and Glossop, Catherine and
          Shah, Dhruv and Levine, Sergey},
  year={2026},
  booktitle     = {IEEE International Conference on Robotics and Automation (ICRA)}
}

@article{kim2026plcbf,
  title={Policy Library CBF: Finite-Horizon Safety at Runtime via Parallel Rollouts},
  author={Kim, Taekyung and Okamoto, Hideki and Hoxha, Bardh and Fainekos, Georgios and Panagou, Dimitra},
  journal={arXiv preprint arXiv:2605.16588},
  year={2026}
}
\end{document}